\documentclass{article}

\usepackage{microtype}
\usepackage{graphicx}
\usepackage{subfigure}
\usepackage{booktabs} % for professional tables

\usepackage{hyperref}

\usepackage{amsmath}
\usepackage{amsmath}
\usepackage{bm}
\usepackage{amsthm}

\newtheorem{theorem}{Theorem}[section]

\usepackage{amsmath, amsthm, amssymb}
\newtheorem{lem}{Lemma}
\usepackage{tkz-euclide}
\usetikzlibrary{calc}

\usepackage[accepted]{icml2018}

\icmltitlerunning{\tiny On the Information Theoretic Distance Measures and Bidirectional Helmholtz Machines}

\begin{document}

\twocolumn[
\icmltitle{On the Information Theoretic Distance Measures and \\ Bidirectional Helmholtz Machines}




\begin{icmlauthorlist}
\icmlauthor{M. Azarafrooz}{cylance}
\icmlauthor{X. Zhao}{cylance}
\icmlauthor{S. Akhavan Masouleh}{cylance}
\end{icmlauthorlist}

\icmlaffiliation{cylance}{Department of Research and Intelligence, Cylance, Irvine, CA}

\icmlcorrespondingauthor{M Azarafrooz}{mazarafrooz@cylance.com}

%\icmlcorrespondingauthor{Eee Pppp}{ep@eden.co.uk}

\icmlkeywords{Machine Learning, ICML}

\vskip 0.3in
]



\printAffiliationsAndNotice{}  % leave blank if no need to mention equal contribution

\begin{abstract}
By establishing a connection between bi-directional Helmholtz machines and information theory, we propose a \textit{generalized} Helmholtz machine. 
Theoretical and experimental results show that given \textit{shallow} architectures, the generalized model outperforms the previous ones substantially.

\end{abstract}

\section{Bidirectional Helmholtz Machines}
Various deep architectures have been proposed for both discriminative and generative models. We consider a generative model as a model that generates new observations, whereas a discriminative model is a model to estimate latent variables from a given set of already-existing observations. In particular, we are interested in the Bidirectional Helmholtz Machine (BIHM) architecture proposed by [1][2], where the architecture includes both a top-down generative model as well as a bottom-up discriminative model (the inverse of the generative model). The architecture assumes a shared-variable hierarchical structure between the generative model and the discriminative model where both models share L layers of discrete latent variables. In particular, layer "i" has $h_{i}$ discrete latent variables where "i" changes from $1$ to $L$ . We then define our generative network as:

\begin{equation}\label{p-factorization}
\small
p(\bold{x},\bold{h}_{1:L})=p(\bold{h}_L)p(\bold{h}_{L-1}|\bold{h}_L)...p(\bold{h}_1|\bold{h}_2)p(\bold{x}|\bold{h}_1)
\end{equation}

We define our discriminative model as:

\begin{equation}\label{q-factorization}
\small
q(\bold{x},\bold{h}_{1:L})=q(\bold{x})q(\bold{h}_1|\bold{x})q(\bold{h}_{2}|\bold{h}_1)...q(\bold{h}_L|\bold{h}_{L-1})
\end{equation}

The generative and discriminative networks defined above are expected to result in the same joint probability distributions (i.e. one may expect the K-L distance between $p(x, h_1, h_2, \dots, h_L)$ and $q(x, h_1, h_2, \dots, h_L)$ to converge to $0$ almost surely as the number of data points increases). 

\subsection{Information Theoretic Distance Measures}
In order to achieve a generative model capable of generating data as close to the distribution of real data as possible, one needs to maximize the likelihood. In particular, consider the marginal probability $\log p(\bold{x})$ where all the intermediate layers are integrated out. In this case, one can consider maximizing the marginal log likelihood $\log p(\bold{x})$.  $\log p(\bold{x})$ is not tractable and hence, the optimization of the marginal likelihood is not trivial. Alternatively, one may use the notion of distance measures borrowed from Information Theory to define a lower bound for the log likelihood. In this case, instead of directly maximizing the log likelihood, one may maximize the lowerbound. In particular, a distance measure is ``information theoretic'' if it follows data processing inequality (DPI) indicating that no post-processing can increase the information content of a signal [3][4][5]. There have been many information theoretic distance measures proposed for the generative models [6]. 

Prior to introducing a distance measure, we first introduce the three probability measures of false alarms ($P_F$), misses ($P_M$), and Chernoff average error ($P_e$). False alarm probability ($P_F$) is the probability that the proposed generative model assigns non-zero probability to regions that are not in reality part of the true target feature space. Misses probability ($P_e$) is the opposite of false alarm probability and is the probability that the proposed generative model assigns zero probability to regions that are actually legitimate regions in the target feature space. Finally, Chernoff average error probability ($P_e$) is the average of false-alarm and misses probabilities for every region of the feature space. It is of interest to appropriately control false alarms, misses, and Chernoff probabilities for any proposed distance measure. To do so, we customized Stein's lemma for our particular problem of interest in order to be used as a theoretical foundation to appropriately control $P_F$, $P_M$, and $P_e$ probability measures for any given information theoretic distance. 

%We use the following Lemma to improve the Bidirectional Helmholtz machines to more generalized version which we will refer to as ``Chernoff Bidirectional Helmholtz machines". 

\begin{lem}[Stein's Lemma]\label{Lemma:Stein}
Consider $\bold{h}_{1:L}$ in our proposed BIHM architecture as a collection of independent and identically distributed (iid) hidden random variables. The distribution of top-down generative model follow joint probability $p(\bold{x}, \bold{h}_{1:L})$. For the the bottom-up discriminative model given the observed data x, the distribution follows $q(\bold{h}_{1:L}|\bold{x})$. In this situation, Stein lemma indicates that the optimal likelihood will result in probabilities that obey the following asymptotic [5][7]:

 \begin{equation}
 \small \lim_{n \rightarrow \infty} \frac{\log P_F}{n}=-\mathcal{D}(p(\bold{x}, \bold{h}_{1:L}),q(\bold{h}_{1:L}|\bold{x})), ~~ \text{fixed}  ~~ P_M
 \end{equation}
 
  \begin{equation}
\small  \lim_{n \rightarrow \infty} \frac{\log P_e}{n}=-\mathcal{C}(p(\bold{x}, \bold{h}_{1:L}),q(\bold{h}_{1:L}|\bold{x}))
 \end{equation}
 
  \begin{equation}\label{B-inequality}
\small \lim_{n \rightarrow \infty} \frac{\log P_e}{n} \leq -\mathcal{B}(p(\bold{x}, \bold{h}_{1:L}),q(\bold{h}_{1:L}|\bold{x}))
 \end{equation}

where $n$ is the size of the training data. $\mathcal{C}$, $\mathcal{B}$ and $\mathcal{D}$ are the Chernoff distance, the Bhattacharyya distance, and the K-L distance respectively. In particular, these distance metrics for our specific problem are defined as:

\end{lem}
\begin{equation}\label{Chernoff}
\small \mathcal{C}(q,p)=\displaystyle \min_{ 0 \leq t \leq 1} -\log E_{q} (p/q)^{t}
 \end{equation}
 
\begin{equation}\label{Bhattacharyya}
 \small \mathcal{B}(p,q)=-\log E_{q} (p/q)^{1/2}
 \end{equation}

\begin{equation}\label{K-L}
 \small \mathcal{D}(q,p)=-E_{q} \log(p/q)
 \end{equation}

Lemma \ref{Lemma:Stein} implies that Chernoff and K-L distance metrics are exponentially related to the optimal likelihood. Further, both Chernoff and K-L metrics are asymmetric, that is $\mathcal{C}(p,q)\neq \mathcal{C}(q,p)$ and $\mathcal{D}(p,q)\neq \mathcal{D}(q,p)$. Other distance metrics can be devised as well. For example [5] shows that $\mathcal{C}(.,.)$ is bounded between the Bhattacharyya distance and Resistor distance defined as: $ \mathcal{R}(q,p) = \frac{\mathcal{D}(q,p)\mathcal{D}(p,q)}{\mathcal{D}(q,p)+\mathcal{D}(p,q)}$. The relation between various information theoretic measures are depicted in Fig. 1. A bidirectional Helmholtz Machine architecture can be generalized to any of these information theoretic measures. [2] proposed the Bhattacharyya bi-directional Helmholtz Machine.

\begin{theorem}\label{Thm:BIHM-error}
Given an infinitely deep Bhattacharyya bi-directional Helmholtz Machine (BIHM) architecture, the Chernoff error probability for the whole network decays exponentially as the amount of training data increases, i.e: $\lim_{L \rightarrow \infty}  \lim_{n \rightarrow \infty} \frac{\log P_{e_{1:L}}}{n} = 0$.
\end{theorem}

\begin{proof} Using lemma \ref{Lemma:Stein} and the Markovian bottom-up dependency structure of the problem ($\bold{h}_1 \rightarrow \bold{h}_2 \rightarrow ... \rightarrow \bold{h}_L$), proof is straight-forward. 

%\begin{equation}\label{B-inequality-all}
%\begin{array}{l}
%\lim_{L \rightarrow \infty} P_{e_{1:L}}=\displaystyle \prod_{l=1}^{L} P_{e_l} \\
 %%\l\lim_{L \rightarrow \infty}  (\lim_{n \rightarrow \infty} \frac{ P_{e_{1:L}}}{n}) \leq \l\lim_{L \rightarrow \infty}  -\exp(L \mathcal{B}(p(\bold{x}, \bold{h}_{1:L}),q(\bold{h}_{1:L}|\bold{x})))=0
 %\end{array}
% \end{equation}
\end{proof}
The authors in [2][9] observed that Bhattacharyya BIHM prefers deeper architectures. This is easily justified by theorem I.1. In contrary, Chernoff BIHM does not have any requirement on the depth of the network, however, it suffers computationally due to an additional minimization over the hyper-parameter $t$. Further, Chernoff distance is not additive indicating that the hyper-parameter $t$ should be optimized per layer of the network. Bhattacharyya, on the other side, is an additive distance metric allowing for an efficient distributed optimization over the layers of the network.

To illustrate this more formally consider Jensen's inequality applied to the concave function $f(\bold{x})=\bold{x}^t,  ~ 0 \leq t \leq 1$. It yields in various probability distance metric between $p(\bold{x}, \bold{h})$ and $q(\bold{h}|\bold{x})$ including the following:
\begin{equation}\label{Jensen-inequality1}
\begin{array}{l}
p(\bold{x})^t= (E_{q} p(\bold{x},\bold{h}_{1:l})/q(\bold{h}_{1:l}|\bold{x}))^t \\\geq \displaystyle \min_{ 0 \leq t \leq 1} -E_{q} (p(\bold{x}, \bold{h}_{1:l})/q(\bold{h}_{1:l}|\bold{x}))^{t}
 \end{array}
 \end{equation}
 \normalsize
%If we substitute $t=1/2$ in Eq. \ref{Jensen-inequality1} we end up with the Bhattacharyya distance defined as:
 %\begin{equation}\label{Bhattacharyya}
 %\mathcal{B}(p(\bold{x}, \bold{h}_{1:l}),q(\bold{h}_{1:l}|\bold{x}))=-\log E_{q} (p(\bold{x}, \bold{h}_{1:l})/q(\bold{h}_{1:l}|\bold{x}))^{1/2}
 %\end{equation}

One can break the right-hand side of eq. \ref{Jensen-inequality1} by taking the log and using eq. \ref{p-factorization}-\ref{q-factorization}, as the following:

\begin{equation}
\begin{array}{l}
 -\displaystyle \min_{0 \leq t_1 \leq 1} \log E_{q(\bold{h}_1|\bold{x})} (\frac{p(\bold{x},\bold{h}_1|\bold{h}_2)}{q(\bold{h}_1|\bold{x})})^{t_1} \\- 
\displaystyle \min_{0 \leq t_2 \leq 1} \log E_{q(\bold{h}_2|\bold{h}_1)}(\frac{p(\bold{h}_2|\bold{h}_3)}{q(\bold{h}_2|\bold{h}_1)})^{t_2}- \\ ...\displaystyle \min_{0 \leq t_L \leq 1} \log E_{q(\bold{h}_L|\bold{h}_{L-1})}(\frac{p(\bold{h}_L)}{q(\bold{h}_L|\bold{h}_{L-1})})^{t_L} 
\leq \log p(\bold{x})
\end{array}
 \end{equation}
\normalsize

As it can be seen, $t_l$ should be adjusted for each layer $l$ separately. This is also know as non-additive property of Chernoff distance.
Additive property implies that the distance between two joint distributions of statistically independent, identically distributed  (i.i.d) random
variables equals the sum of the marginal distances. However due to extra optimization terms $t_l$ per each layer of the networks, the Chernoff distance is not additive.

\begin{figure}
\begin{center}
\begin{tikzpicture}[scale=5,extended line/.style={shorten >=-#1,shorten <=-#1},]
\draw [thin, dashed] (0,0) grid (1,1);
% Euclidean
\draw [->, thick](0,0)--(0,1.2) node[right]{$-\log \mathcal{Z} (t)$};
\draw [->, thick](0,0)--(1.2,0) node[right]{$t$};
% polar coordinate
\draw (0,0)--(1,1)node[anchor=south west]{};
\draw [thin,dashed] (1.0/3,1.0/3)--(1.0/3,0);
\draw [thin,dashed] (0,1.0/3)--(1.0/3,1.0/3);
\draw [thin,dashed] (0, 0.18)--(0.5, 0.18);
\draw [thin,dashed] (0.5, 0)--(0.5, 0.18);
\draw [black,thin,dashed] (0, 0.2)--(0.45, 0.2);
\draw [thin,dashed] (0.45, 0)--(0.45, 0.2);
\fill [black](0.5,0.18) circle(0.5pt);
\fill [black](1.0/3,1.0/3) circle(0.5pt);
\fill [black](0.5,0.18) circle(0.5pt);
\fill [black](0.45,0.2) circle(0.5pt);
\fill [black](0,0.165) circle(0.5pt);
\fill [black](0,0.2) circle(0.5pt);
\fill [black](0,1.0/3) circle(0.5pt);
\fill [black](0,0.5) circle(0.5pt);
\fill [black](0,0.75) circle(0.5pt);
\fill [black](0,1.0) circle(0.5pt);
\fill [black](1,1) circle(0.5pt);
\draw (0.5,0.6)node[anchor=west,rotate=45]{$t\mathcal{D}(p || q)$};
\draw (0.53,0.31)node[anchor=west,rotate=-25]{$(1-t)\mathcal{D}(q || p)$};
\draw (-0.28, 0.155) node[anchor=west]{$\mathcal{B}(p,q)$};
\draw (-0.28, 0.21) node[anchor=west]{\color{black} \boldsymbol{$\mathcal{C}(p,q)$}};
\draw (-0.28, 1.0/3) node[anchor=west]{$\mathcal{R}(p,q)$};
\draw (-0.28, 0.5) node[anchor=west]{$\mathcal{D}(q||p)$};
\draw (-0.28, 0.75) node[anchor=west]{$\mathcal{J}(p,q)$};
\draw (-0.28, 1) node[anchor=west]{$\mathcal{D}(p||q)$};
\coordinate (G) at (0,0);
\coordinate (R) at (1,0);
\draw [thick] (G) to[out=45, in=155,distance=0.45cm] (R);
% draw ticks and its labels
\foreach \x/\xtext in {0/0, 0.333/t', 0.45/t'', 0.5/\frac{1}{2}, 1/1}
{\draw (\x cm,1pt ) -- (\x cm,-1pt ) node[anchor=north] {$\xtext$};}
\draw [thick] (0,0.5)--(1,0) node[above]{};  % method 1
                                                % method 2
\end{tikzpicture}
\end{center}
\caption{Relations among various relation Information Theoretic Metrics [5]}
\end{figure}
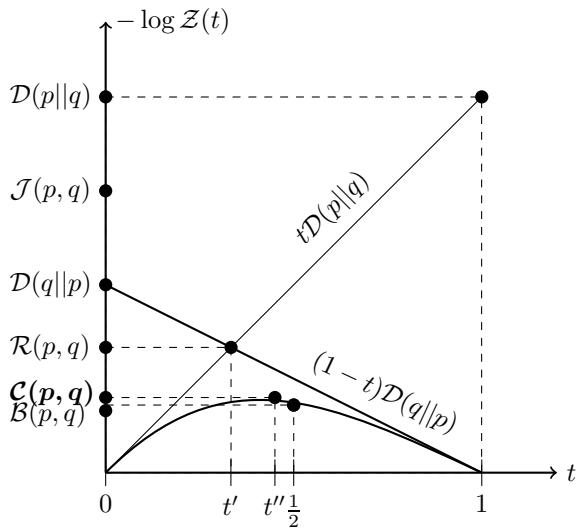
\subsection{Estimating the bounds and gradients via Importance sampling }
So far, we were able to propose valid lower-bounds in order to optimize our intractable log likelihood. However, our bounds are still not easily tractable. As proposed by [8][9][10], one can use importance sampling (IS) to estimate both the bounds as well as their gradients. 

To provide a distance agnostic training algorithm for any generalized bi-directional Helmholz machines such as $\mathcal{C}$-BIHM then consider an arbitrary bound such as $\mathcal{L}(\theta, \bold{h}, \bold{x} \sim D)$ with parameter $\theta$, latent variables $\bold{h}$ and data sample $\bold{x}$. The gradient terms are as the following:
\begin{equation}\label{eq:IS-estimator}
 \begin{array}{l}
\frac{\partial}{\partial \theta} \mathcal{L}(\theta, \bold{h}, \bold{x})=\frac{1}{p(\bold{x})} E_{\bold{h} \sim q(\bold{h}|\bold{x})} [\frac{p(\bold{x},\bold{h})}{q(\bold{h}|\bold{x})} \frac{\partial}{\partial \theta} \mathcal{L}(\theta, \bold{h}, \bold{x})]\\
\simeq \sum_{k=1}^K \tilde{w}_k \frac{\partial}{\partial \theta} \mathcal{L}(\theta, \bold{h}^k, \bold{x}),\\
 \bold{h}^k \sim q(\bold{h}|\bold{x}),\tilde{w}_k= \frac{w_k}{\sum_{k' } w_{k'}}, w_k=\frac{p(\bold{x}, \bold{h}^k)}{q(\bold{h}^k|\bold{x})}
\end{array}
 \end{equation} 

\section{Simulation Results}
The Distance measure $\mathcal{C}$ is bounded between the measures $\mathcal{B}$ and $\mathcal{R}$ (Fig.1). Hence, we can approximate $\mathcal{C}$ as the weighted average of $\mathcal{B}$ and $\mathcal{R}$. That is,  $\mathcal{C} \approx \alpha\mathcal{B}+(1-\alpha)\mathcal{R}$, where $\alpha$ is the weight hyper-parameter. We replicated the binary MNIST experiment proposed in [2] with $\mathcal{C}$ measure estimated as the weighted average of $\mathcal{B}$ and $\mathcal{R}$. Simulation results show that architectures with $\mathcal{C}$ can reach an equally well log-likelihood compared to the reported result of 84.3 in [2], within the first 70 epochs, using much fewer layers. This finding confirms the result of Lemma \ref{Lemma:Stein} and Theorem \ref{Thm:BIHM-error}. Table 1 compares the results of these two architectures for 2 layers (with 300,10 hidden units) and 12 layers (with 300,200,100,75,50,35,30,25,20,15,10,10 hidden units). Note that $\mathcal{C}$ architecture outperforms $\mathcal{B}$ substantially within the first 70 epochs for the shallow architectures. The number of particle used in the importance sampling is 100 across all the experiments, learning rate is 0.001 with $L_1$ regularization of $10^{-3}$  . $\alpha=0.2$ is used for approximating $\mathcal{C}$. 

\section{Conclusion}
These simulation and theoretical results shed light on the mystery behind the preference of $\mathcal{B}$-BIHM for deeper architectures. Moreover, it shows that an ideal distance metric for BIHM architectures is Chernoff distance $\mathcal{C}$.  We argued that although Chernoff $\mathcal{C}$-BIHM does not have any requirement on the depth of the network, but due to its \textit{non-additive} property, it does not yield to an efficient distributed optimization over the layers of the network. One can instead approximate $\mathcal{C}$ by weighted combination of two other \textit{additive} distance metric of Bhattacharyya $\mathcal{B}$ and Resistor distance $\mathcal{R}$.  

One future work is to evaluate the efficiency of importance sampling under aforementioned distance metrics.

\begin{table}
\caption{The valid negative log-likelihood for binary mnist experiment after 70 epochs.}
\begin{center}
\begin{tabular}{ |c|c|c| } 
 \hline
 Distance measure & \textit{2 Layers} & 12 Layers \\
 \hline
 \hline
 $\mathcal{B}$-BIHM & \textit{94.29} & 88.46 \\ 
 \hline
 approximated $\color{black}\mathcal{C}$-BIHM & \textbf{84.97} & \textbf{84.73} \\ 
 \hline
\end{tabular}
\end{center}
\end{table}

%Figure 2. shows the log-likelihood convergence plot for the experiment
%Since $\mathcal{C}$ is bounded between $\mathcal{B}$ and $\mathcal{R}$ (as it is shown in Fig. 1), we propose to approximate $\mathcal{C}$ as $\alpha\mathcal{B}+(1-\alpha)\mathcal{R}$ where $\alpha$ is a hyper-parameter. We conducted the binary MNIST experiment of [2] with their exact setup (everything fixed except the distance measure). The convergence plot for the log-likelihood is shown in Fig. 2. Our simulations shows that even shallow architectures achieve considerable results whereas Bhattacharyya requires deeper architectures. This confirms the result of the Lemma \ref{Lemma:Stein} and Theorem \ref{Thm:BIHM-error}. A side to side comparisons is reported result in Table 1 for 2 layers 300,10 hidden units.

%\begin{figure}[t]
%\centerline{\includegraphics[scale=0.3]{/Users/mahdiazarafrooz/Desktop/500-batch-big.png}
%}
%\caption{bmnist generated data}
%\label{fig:bmnits}
%\end{figure}

%\begin{figure}
%\end{figure}

\section{References}
%	REFERENCE LIST
%----------------------------------------------------------------------------------------
%\phantomsection
%\bibliographystyle{unsrt}
%\bibliography{sample}

[1] Dayan, Hinton, Frey, Neal, Zemel. (1995) Helmholtz machine and Wake/Sleep algorithm.

[2]  Jorg Bornschein, Samira Shabanian, Asja Fischer, Yoshua Bengio (2016) Bidirectional Helmholtz Machines {\it Proceedings of The 33rd International Conference on Machine Learning}, pp.\ 2511--2519, 2016.

[3] Thomas M. Cover and Joy A Thomas. Elements of Information Theory (Wiley Series in Telecommunications
and Signal Processing). Wiley-Interscience, 2006.

[4] N. Tishby, F. Pereira, W. Bialek, The information bottleneck method, in:  Proc. 37th Allerton Conference and Communication Control and Computing, 1999, pp. 368-377.

[5] Sinan Sinanovic and Don H. Johnson. (2007). Toward a theory of information processing. {\it Signal Process. 87, 6 , 1326-1344. DOI=http://dx.doi.org/10.1016/j.sigpro.2006.11.005}

[6] Lucas Theis, Aaron van den Oord, and Matthias Bethge (2015). A note on the evaluation of generative models. {\it arXiv preprint arXiv:1511.01844}.

[7] H. Chernoff (1956). Large-sample theory: Parametric case. {\it Ann. Math. Stat., 27:1-22}

[8] Bornschein, Jorg and Bengio, Yoshua (2015) Reweighted wake sleep. {\it In International Conference on Learning Representations (ICLR).}

%[7] H. Chernoff. Measure of asymptotic efficiency for tests of a hypothesis based on the sum of observations (1952.). {\it Ann. Math. Stat., 23:493--507.}

[9] S. Webb and Y.W. Teh. (2016) A Tighter Monte Carlo Objective with Renyi $\alpha$-divergence Measures {\it Workshop on Bayesian Deep Learning, NIPS 2016, Barcelona, Spain.}

[10]  Yuri Burda, Roger Grosse, and Ruslan Salakhutdinov (2016). Importance weighted autoencoders. {\it arXiv preprint arXiv:1509.00519v3 [cs.LG]}.

\end{document}